\documentclass{svjour3}
\smartqed

\usepackage[utf8]{inputenc}
\usepackage[T1]{fontenc}
\usepackage[english]{babel}
\usepackage{graphicx}

\usepackage{textcomp}

\usepackage{setspace}

\usepackage{xcolor}

\usepackage{hyperref}
\hypersetup{
  bookmarksnumbered,
  bookmarksopen,
  bookmarksopenlevel=0,
  colorlinks=false,
    citebordercolor=teal,
    linkbordercolor=pink,
	pdfauthor={Evgenij Thorstensen},
	pdftitle={Structural Decompositions for Problems with Global Constraints},
	pdfsubject={Constraint Programming},
	pdfkeywords={CS,constraints,databases,tractability},
   implicit=false,
}

\usepackage{verbatim}

\usepackage{complexity}

\usepackage{mathtools}
\usepackage{amsfonts}
\usepackage[amsmath,hyperref]{ntheorem}

\usepackage[nospace,noadjust]{cite}

\makeatletter \let\cl@chapter\relax \makeatother

\usepackage{cleveref}

\usepackage{algorithm}
\usepackage{algpseudocode}

\crefname{section}{Section}{Sections}
\crefname{subsection}{Section}{Sections}

\crefname{definition}{Definition}{Definitions}
\crefname{theorem}{Theorem}{Theorems}
\crefname{lemma}{Lemma}{Lemmata}
\crefname{propositions}{Proposition}{Propositions}
\crefname{corollary}{Corollary}{Corollaries}
\crefname{property}{Property}{Properties}
\crefname{conjecture}{Conjecture}{Conjectures}

\crefname{example}{Example}{Examples}
\crefname{problem}{Problem}{Problems}
\crefname{notation}{Notation}{Notation}
\crefname{algo}{Algorithm}{Algorithms}
\crefname{algorithm}{Algorithm}{Algorithms}
\crefname{figure}{Figure}{Figures}

\title{Structural Decompositions for Problems\\with Global Constraints}

\author{
Evgenij Thorstensen
}

\authorrunning{E. Thorstensen}

\institute{
E. Thorstensen \at Department of Informatics, University of Oslo, Norway\\
\email{evgenit@ifi.uio.no}
}

\date{}

\newcommand{\tw}{\ensuremath{\mathsf{tw}}}

\newcommand{\hw}{\ensuremath{\mathsf{ghw}}}

\newcommand{\fhw}{\ensuremath{\mathsf{fhw}}}
\newcommand{\subw}{\ensuremath{\mathsf{subw}}}

\newcommand{\vars}{\ensuremath{\mathcal{V}}}
\newcommand{\sol}{\ensuremath{\mathsf{sol}}}

\newcommand{\CSP}{\ensuremath{\textup{CSP}}}

\newcommand{\hyp}{\ensuremath{\mathsf{hyp}}}

\newcommand{\iv}{\ensuremath{\mathsf{iv}}}

\newcommand{\ic}{\ensuremath{\mathsf{ic}}}

\newcommand{\Pj}{\ensuremath{\mathsf{pj}}}

\newcommand{\Extlang}{\ensuremath{\mathbf{Ext}}}
\DeclarePairedDelimiter{\tup}{\langle}{\rangle}

\newcommand{\languageword}{catalogue}
\newcommand{\languagesymbol}{\ensuremath{\Gamma}}

\newcommand{\cost}{\ensuremath{\mathsf{cost}}}
\newcommand{\WCSP}{\ensuremath{\textup{WCSP}}}

\makeatletter
\def\blfootnote{\xdef\@thefnmark{}\@footnotetext}
\makeatother

\begin{document}

\maketitle

\begin{abstract}
  A wide range of problems can be modelled as constraint satisfaction
  problems (CSPs), that is, a set of constraints that must be
  satisfied simultaneously. Constraints can either be represented
  extensionally, by explicitly listing allowed combinations of values,
  or implicitly, by special-purpose algorithms provided by a
  solver. 

  Such implicitly represented constraints, known as global
  constraints, are widely used; indeed, they are one of the key
  reasons for the success of constraint programming in solving
  real-world problems. In recent years, a variety of restrictions on
  the structure of CSP instances have been shown to yield tractable
  classes of CSPs. However, most such restrictions fail to guarantee
  tractability for CSPs with global constraints. We therefore study
  the applicability of structural restrictions to instances with such
  constraints.

  We show that when the number of solutions to a CSP instance is
  bounded in key parts of the problem, structural restrictions can be
  used to derive new tractable classes. Furthermore, we show that this
  result extends to combinations of instances drawn from known
  tractable classes, as well as to CSP instances where constraints
  assign costs to satisfying assignments.

\end{abstract}

\blfootnote{A preliminary version of this paper appeared in
  \textsl{Proceedings of the 19th International Conference on
    Principles and Practice of Constraint Programming (CP
    2013)}.}

\section{Introduction}

Constraint programming (CP) is widely used to solve a
variety of practical problems such as planning and scheduling
\cite{vanHoeve2006:global-const,Wallace96practicalapplications}, and
industrial configuration
\cite{pup-cpaior2011,bin-repack-datacenters}. Constraints can either
be represented explicitly, by a table of allowed assignments, or
implicitly, by specialized algorithms provided by the constraint
solver. These algorithms may take as a parameter a \emph{description}
that specifies exactly which kinds of assignments a particular
instance of a constraint should allow. Such implicitly represented
constraints are known as global constraints, and a lot of the success
of CP in practice has been attributed to solvers providing global
constraints \cite{Rossi06:handbook,Gent06:minion,wallace97-eclipse}.

The theoretical properties of constraint problems, in particular the
computational complexity of different types of problem, have been
extensively studied and quite a lot is known about what restrictions
on the general \emph{constraint satisfaction problem} are sufficient
to make it tractable
\cite{struct-decomp-stacs-et,bulatov05:classifying-constraints,CohenUnifTract,compar-decomp-csp,grohe-hom-csp-complexity,DBLP:journals/corr/abs-0911-0801}. In
particular, many structural restrictions, that is, restrictions on how
the constraints in a problem interact, have been identified and shown
to yield tractable classes of CSP instances
\cite{Gottlob02:jcss-hypertree,grohe-marx-frac-edge-cover,DBLP:journals/corr/abs-0911-0801}. However,
much of this theoretical work has focused on problems where each
constraint is explicitly represented, and most known structural
restrictions fail to yield tractable classes for problems with global
constraints. This is the case even when the global constraints are
fairly simple, such as overlapping difference constraints with acyclic
hypergraphs \cite{Kutz08:sim-match}.

Theoretical work on global constraints has to a large extent focused
on developing efficient algorithms to achieve various kinds of local
\emph{consistency} for individual constraints.  This is generally done
by pruning from the domains of variables those values that cannot lead
to a satisfying assignment
\cite{Walsh07:global,Samer11:constraints-tractable}.  Another strand
of research has explored conditions that allow global
constraints to be replaced by collections of explicitly represented constraints
\cite{bess-nvalue}. These techniques allow faster implementations of
algorithms for \emph{individual constraints}, but do not shed much
light on the complexity of problems with multiple \emph{overlapping}
global constraints, which is something that practical problems
frequently require.

As such, in this paper we investigate the properties of explicitly
represented constraints that allow structural restrictions to
guarantee tractability. Identifying such properties will allow us to
find global constraints that also possess them, and lift structural
restrictions to instances with such constraints.

As discussed in~\cite{ChenGroheCompactRel}, when the constraints in a
family of problems have unbounded arity, the way that the constraints
are {\em represented} can significantly affect the complexity.
Previous work in this area has assumed that the global constraints
have specific representations, such as
propagators~\cite{Green08:cp-structural}, negative
constraints~\cite{gdnf-cohen-repr}, or GDNF/decision
diagrams~\cite{ChenGroheCompactRel}, and exploited properties
particular to that representation. In contrast, we will use a
definition of global constraints, used also in \cite{cp-tract-comb},
that allows us to discuss different representations in a uniform
manner. Armed with this definition, we obtain results that rely on a
relationship between the size of a global constraint and the number of
its satisfying assignments.

Furthermore, as our definition is general enough to capture arbitrary
problems in \NP, we demonstrate how our results can be used to
decompose a constraint problem into smaller constraint problems (as
opposed to individual constraints), and when such decompositions lead
to tractability. The results that we obtain on this topic extend
previous research by Cohen and Green \cite{guarded-decomp-CG}. In
addition to being more general, our results arguably use simpler
theoretical machinery.

Finally, we show how our results can be extended to \emph{weighted
  CSP} \cite{gottlob-etal-tractable-optimization,degivry06}, that is,
CSP where constraints assign costs to satisfying assignments, and the
goal is to find an optimal solution.

\section{Preliminaries}

In this section, we define the basic concepts that we will use
throughout the paper. In particular, we give a precise definition of
global constraints and of structural decompositions.

\subsection{Global Constraints}

\begin{definition}[Variables and assignments]
  Let $V$ be a set of variables, each with an associated finite set of
  domain elements. We denote the set of domain elements (the domain)
  of a variable $v$ by $D(v)$.  We extend this notation to arbitrary
  subsets of variables, $W$, by setting $D(W) =
  \displaystyle\bigcup_{v \in W} D(v)$.

  An {\em assignment} of a set of variables $V$ is a function $\theta
  : V \rightarrow D(V)$ that maps every $v \in V$ to an element
  $\theta(v) \in D(v)$. We write $\vars(\theta)$ for the set of
  variables $V$.

  We denote the restriction of $\theta$ to a set of variables $W
  \subseteq V$ by $\theta|_W$. We also allow the special assignment
  $\bot$ of the empty set of variables. In particular, for every
  assignment $\theta$, we have $\theta|_\emptyset = \bot$.
\end{definition}

\begin{definition}[Projection]
  Let $\Theta$ be a set of assignments of a set of variables $V$. The
  \emph{projection} of $\Theta$ onto a set of variables $X \subseteq
  V$ is the set of assignments $\pi_X(\Theta) = \{\theta|_X \mid
  \theta \in \Theta\}$.
\end{definition}
   
Note that when $\Theta = \emptyset$ we have $\pi_X(\Theta) =
\emptyset$, but when $X = \emptyset$ and $\Theta \neq \emptyset$, we
have $\pi_X(\Theta) = \{\bot\}$.

\begin{definition}[Disjoint union of assignments]
 \label{def:disjoint-union}
 Let $\theta_1$ and $\theta_2$ be two assignments of disjoint sets of
 variables $V_1$ and $V_2$, respectively. The \emph{disjoint union} of
 $\theta_1$ and $\theta_2$, denoted $\theta_1 \oplus \theta_2$, is the
 assignment of $V_1 \cup V_2$ such that $(\theta_1 \oplus \theta_2)(v)
 = \theta_1(v)$ for all $v \in V_1$, and $(\theta_1 \oplus
 \theta_2)(v) = \theta_2(v)$ for all $v \in V_2$.
\end{definition}

Global constraints have traditionally been defined, somewhat vaguely,
as constraints without a fixed arity, possibly also with a compact
representation of the constraint relation. For example, in
\cite{vanHoeve2006:global-const} a global constraint is defined as ``a
constraint that captures a relation between a non-fixed number of
variables''.

Below, we offer a precise definition similar to the one in
\cite{Walsh07:global}, where the authors define global constraints for
a domain $D$ over a list of variables $\sigma$ as being given
intensionally by a function $D^{|\sigma|} \rightarrow \{0, 1\}$
computable in polynomial time. Our definition differs from this one in
that we separate the general {\em algorithm} of a global constraint
(which we call its {\em type}) from the specific description.  This
separation allows us a better way of measuring the size of a global
constraint, which in turn helps us to establish new complexity
results.

\begin{definition}[Global constraints]
  \label{def:glob-const}
  A \emph{global constraint type} is a parameterized polynomial-time
  algorithm that determines the acceptability of an assignment of a
  given set of variables.

  Each global constraint type, $e$, has an associated set of
  \emph{descriptions}, $\Delta(e)$. Each description $\delta \in
  \Delta(e)$ specifies appropriate parameter values for the algorithm
  $e$.  In particular, each $\delta \in \Delta(e)$ specifies a set of
  variables, denoted by $\vars(\delta)$. We write $|\delta|$ for the
  number of bits used to represent $\delta$.

  A \emph{global constraint} $e[\delta]$, where $\delta \in
  \Delta(e)$, is a function that maps assignments of $\vars(\delta)$
  to the set $\{0,1\}$.  Each assignment that is allowed by
  $e[\delta]$ is mapped to 1, and each disallowed assignment is mapped
  to 0.  The \emph{extension} or \emph{constraint relation} of
  $e[\delta]$ is the set of assignments, $\theta$, of $\vars(\delta)$
  such that $e[\delta](\theta) = 1$. We also say that such assignments
  \emph{satisfy} the constraint, while all other assignments
  \emph{falsify} it.
\end{definition}

When we are only interested in describing the set of assignments that
satisfy a constraint, and not in the complexity of determining
membership in this set, we will sometimes abuse notation by writing
$\theta \in e[\delta]$ to mean $e[\delta](\theta) = 1$.

As can be seen from the definition above, a global constraint is not
usually explicitly represented by listing all the assignments that
satisfy it. Instead, it is represented by some description $\delta$
and some algorithm $e$ that allows us to check whether the constraint
relation of $e[\delta]$ includes a given assignment. To stay within
the complexity class \NP, this algorithm is required to run in
polynomial time. As the algorithms for many kinds of global
constraints are built into modern constraint solvers, we measure the
{\em size} of a global constraint's representation by the size of its
description.

\begin{example}[EGC]
  \label{example:egc}
  A very general global constraint type is the \emph{extended global
    cardinality} constraint type~\cite{Samer11:constraints-tractable}.
  This form of global constraint is defined by specifying, for every
  domain element $a$, a finite set of natural numbers $K(a)$, called
  the cardinality set of $a$. The constraint requires that the number
  of variables which are assigned the value $a$ is in the set $K(a)$,
  for each possible domain element $a$.

  Using our notation, the description $\delta$ of an EGC global
  constraint specifies a function $K_\delta : D(\vars(\delta))
  \rightarrow \mathcal P(\mathbb N)$ that maps each domain element to
  a set of natural numbers.  The algorithm for the EGC constraint then
  maps an assignment $\theta$ to $1$ if and only if, for every domain
  element $a \in D(\vars(\delta))$, we have that $|\{v \in
  \vars(\delta) \mid \theta(v) = a\}| \in K_\delta(a)$.

\end{example}

  
  

\begin{example}[Table and negative constraints]
  \label{example:table-const}
  A rather degenerate example of a a global constraint type is 
  the \emph{table} constraint.
  
  In this case the description $\delta$ is simply a list of assignments 
  of some fixed set of variables, $\vars(\delta)$. The algorithm for
  a table constraint then decides, for any
  assignment of $\vars(\delta)$, whether it is included in $\delta$. 
  This can be done in a time which is linear in the size of $\delta$ 
  and so meets the polynomial time requirement. 

  {\em Negative} constraints are complementary to table constraints,
  in that they are described by listing {\em forbidden}
  assignments. The algorithm for a negative constraint $e[\delta]$
  decides, for any assignment of $\vars(\delta)$, whether it is {\em
    not} included in $\delta$. Observe that disjunctive clauses, used
  to define propositional satisfiability problems, are a special case
  of the negative constraint type, as they have exactly one forbidden
  assignment.
  
  We observe that any global constraint can be rewritten as a table or
  negative constraint. However, this rewriting will, in general, incur
  an exponential increase in the size of the description.
\end{example}

As can be seen from the definition above, a table global constraint is
explicitly represented, and thus equivalent to the usual notion of an
extensionally represented constraint.

In some cases, particularly for table constraints, we will make use of
the standard notion of a relational join, which we define below.

\begin{definition}[Constraint join]
  A global constraint $e_j[\delta_j]$ is the join of two global
  constraints $e_1[\delta_1]$ and $e_2[\delta_2]$ whenever
  $\vars(\delta_j) = \vars(\delta_1) \cup \vars(\delta_2)$, and
  $\theta \in e_j[\delta_j]$ if and only if $\theta|_{\vars(\delta_1)}
  \in e_1[\delta_1]$ and $\theta|_{\vars(\delta_2)} \in
  e_2[\delta_2]$.
\end{definition}

\begin{definition}[CSP instance]
  An instance of the constraint satisfaction problem (CSP) is a pair
  $\tup{V, C}$ where $V$ is a finite set of \emph{variables}, and $C$
  is a set of \emph{global constraints} such that $V =
  \displaystyle\bigcup_{e[\delta] \in C} \vars(\delta)$. In a CSP
  instance, we call $\vars(\delta)$ the \emph{scope} of the constraint
  $e[\delta]$.

  A \emph{classic} CSP instance is one where every constraint is a
  table constraint.

  A \emph{solution} to a CSP instance $P = \tup{V, C}$ is an
  assignment $\theta$ of $V$ which satisfies every global constraint,
  i.e., for every $e[\delta] \in C$ we have $\theta|_{\vars(\delta)}
  \in e[\delta]$. We denote the set of solutions to $P$ by $\sol(P)$.

  The \emph{size} of a CSP instance $P = \tup{V, C}$ is $|P| = |V| +
  \displaystyle\sum_{v \in V} |D(v)| + \displaystyle\sum_{e[\delta]
    \in C} |\delta|$.
\end{definition}

Note that this definition disallows CSP instances with variables that
are not in the scope of any constraint. Since a variable that is not
in the scope of any constraint can be assigned any value from its
domain, excluding such variables can be done without loss of
generality. While this condition is strictly speaking not necessary,
it will allow us to simplify some proofs later on. In particular, it
entails that the set of solutions to a CSP instance is precisely the
set of assignments satisfying the constraint obtained by taking the
join of every constraint in the CSP instance.

To illustrate these definitions, consider the connected graph
partition problem (CGP) \cite[p.~209]{Garey79:intractability},
formally defined below. Informally, the CGP is the problem of
partitioning the vertices of a graph into bags of a given size while
minimizing the number of edges that have endpoints in different bags.

\begin{problem}[Connected graph partition (CGP)]
  \label{prob:cgp}
  We are given an undirected and connected graph $\tup{V, E}$, as well
  as $\alpha, \beta \in \mathbb N$. Can $V$ be partitioned into
  disjoint sets $V_1,\ldots,V_m$, for some $m$, with $|V_i| \leq
  \alpha$ such that the set of broken edges $E' = \{\{u, v\} \in E
  \mid u \in V_i, v \in V_j, i\not=j\}$ has cardinality $\beta$ or
  less?
\end{problem} 

\begin{example}[The CGP encoded with global constraints]
  \label{example:cgp-as-csp}
  Given a connected graph $G = \tup{V, E}$, $\alpha$, and $\beta$, we
  build a CSP instance $\tup{A \cup B, C}$ as follows. The set $A$
  will have a variable $v$ for every $v \in V$ with domain $D(v) =
  \{1,\ldots,|V|\}$, while the set $B$ will have a boolean variable
  $e$ for every edge in $E$.

  The set of constraints $C$ will have an EGC constraint $C^{\alpha}$
  on $A$ with $K(i) = \{0,\ldots,\alpha\}$ for every $1 \leq i \leq
  |V|$. Likewise, $C$ will have an EGC constraint $C^\beta$ on $B$
  with $K(0) = \{0,\ldots,|E|\}$ and $K(1) = \{0,\ldots,\beta\}$.

  Finally, to connect $A$ and $B$, the set $C$ will have for every
  edge $\{u, v\} \in E$, with corresponding variable $e \in B$, a
  table constraint on $\{u, v, e\}$ requiring $\theta(u) \not=
  \theta(v) \rightarrow \theta(e)=1$.

  As an example, \cref{fig:cgp-ext-example} shows this encoding for
  the CGP on the graph $C_5$, that is, a simple cycle on five
  vertices.
\end{example}

This encoding follows the definition of \cref{prob:cgp} quite closely,
and can be done in polynomial time.

\begin{figure}[h]
  \centering
  \includegraphics[width=0.9\textwidth]{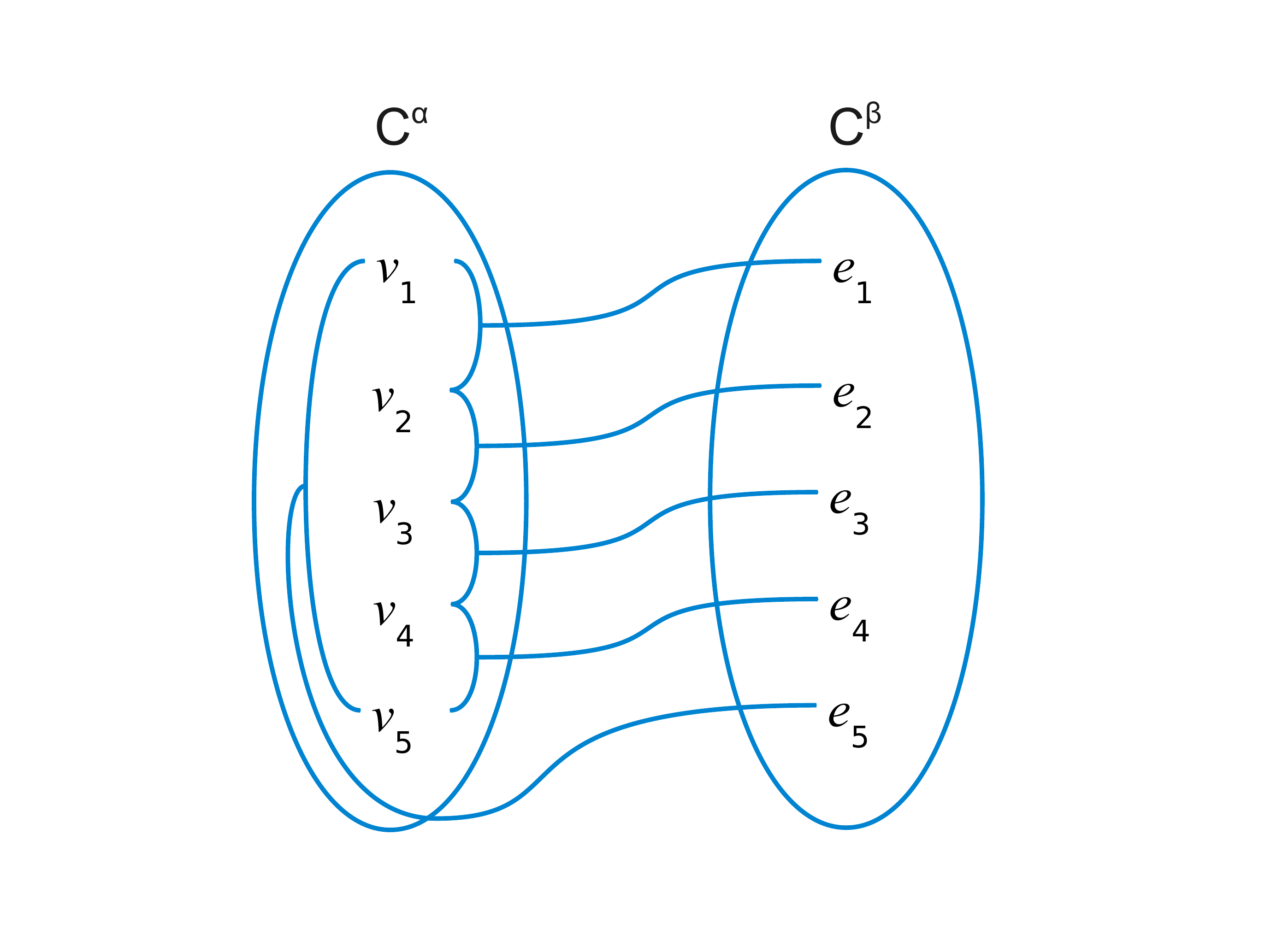}
  \caption{CSP encoding of the CGP on the graph $C_5$.}
  \label{fig:cgp-ext-example}
\end{figure}

\subsection{Structural Restrictions}
\label{sect:structural-restrictions}

In recent years, there has been a flurry of research into identifying
tractable classes of classic CSP instances based on structural
restrictions, that is, restrictions on the hypergraphs of CSP
instances. Below, we present and discuss a few representative
examples. In \cref{sect:prop-ext,sect:subproblem-decompositions}, we
will show how these techniques can be applied to CSP instance with
global constraints. To present the various structural restrictions, we
will use the framework of width functions, introduced by Adler
\cite{adler-thesis}.

\begin{definition}[Hypergraph]
  A hypergraph $\tup{V, H}$ is a set of vertices $V$ together with a
  set of hyperedges $H \subseteq \mathcal P(V)$.

  Given a CSP instance $P = \tup{V, C}$, the hypergraph of $P$,
  denoted $\hyp(P)$, has vertex set $V$ together with a hyperedge
  $\vars(\delta)$ for every $e[\delta] \in C$.
\end{definition}

\begin{definition}[Tree decomposition]
  A \emph{tree decomposition} of a hypergraph $\tup{V, H}$ is a pair
  $\langle T, \chi \rangle$ where $T$ is a tree and $\chi$ is a
  labelling function from nodes of $T$ to subsets of $V$, such that
  \begin{enumerate}
  \item for every $v \in V$, there exists a node $t$ of $T$ such that
    $v \in \chi(t)$,
  \item for every hyperedge $h \in H$, there exists a node $t$ of $T$
    such that $h \subseteq \chi(t)$, and
  \item for every $v \in V$, the set of nodes $\{ t \mid v \in
    \chi(t) \}$ induces a connected subtree of $T$.
  \end{enumerate}
\end{definition}

\begin{definition}[Width function]
  Let $G = \tup{V, H}$ be a hypergraph. A \emph{width function} on $G$
  is a function $f : \mathcal P(V) - \{\emptyset\} \rightarrow \mathbb
  R^+$ that assigns a positive real number to every nonempty subset of
  vertices of $G$. A width function $f$ is monotone if $f(X) \leq
  f(Y)$ whenever $X \subseteq Y$.

  Let $\tup{T, \chi}$ be a tree decomposition of $G$, and $f$ a
  width function on $G$. The \emph{$f$-width} of $\tup{T, \chi}$ is
  $\max(\{f(\chi(t)) \mid t \mbox{ node of } T\})$. The
  \emph{$f$-width} of $G$ is the minimal $f$-width over all its tree
  decompositions.
\end{definition}

In other words, a width function on a hypergraph $G$ tells us how to
assign weights to nodes of tree decompositions of $G$.

\begin{definition}[Treewidth]
  Let $f(X) = |X|-1$. The treewidth $\tw(G)$ of a hypergraph $G$ is
  the $f$-width of $G$.
\end{definition}

Let $G = \tup{V, H}$ be a hypergraph, and $X \subseteq V$. An edge
cover of $X$ is any set of hyperedges $H' \subseteq H$ that satisfies
$X \subseteq \bigcup H'$. The edge cover number $\rho(X)$ of $X$ is
the size of the smallest edge cover of $X$. It is clear that $\rho$ is
a width function.

\begin{definition}[\protect{\cite[Chapter 2]{adler-thesis}}]
  The generalized hypertree width $\hw(G)$ of a hypergraph $G$ is the
  $\rho$-width of $G$.
\end{definition}

Next, we define a relaxation of hypertree width known as fractional
hypertree width, introduced by Grohe and Marx
\cite{grohe-marx-frac-edge-cover}.

\begin{definition}[Fractional edge cover]
  Let $G = \tup{V, H}$ be a hypergraph, and $X \subseteq V$. A
  \emph{fractional edge cover} for $X$ is a function $\gamma : H
  \rightarrow [0, 1]$ such that $\displaystyle\sum_{v \in h \in H}
  \gamma(h) \geq 1$ for every $v \in X$. We call $\displaystyle\sum_{h
    \in H} \gamma(h)$ the weight of $\gamma$. The \emph{fractional
    edge cover number} $\rho^*(X)$ of $X$ is the minimum weight over
  all fractional edge covers for $X$. It is known that this minimum is
  always rational \cite{grohe-marx-frac-edge-cover}. We furthermore
  define $\rho^*(G) = \rho^*(V)$.
\end{definition}

\begin{definition}
  The \emph{fractional hypertree width} $\fhw(G)$ of a hypergraph $G$
  is the $\rho^*$-width of $G$.
\end{definition}

For a class of hypergraphs $\mathcal H$ and a notion of width
$\alpha$, we write $\alpha(\mathcal H)$ for the maximal $\alpha$-width
over the hypergraphs in $\mathcal H$. If this is unbounded we write
$\alpha(\mathcal H) = \infty$; otherwise $\alpha(\mathcal H) <
\infty$.

Bounding any of the above width measures by a constant can be used to
guarantee tractability for classes of CSP instances where all
constraints are table constraints.

\begin{theorem}[\cite{Dalmau02:csp-tractability-cores,Gottlob02:jcss-hypertree,adler-gottlob-grohe-ghw,grohe-marx-frac-edge-cover,Marx2010-approx-frac-htw}]
  \label{thm:width-tract}
  Let $\mathcal H$ be a class of hypergraphs. For every $\alpha \in
  \{\tw, \hw,\fhw\}$, any class of classic CSP instances whose
  hypergraphs are in $\mathcal H$ is tractable if $\alpha(\mathcal H)
  < \infty$.
\end{theorem}

To go beyond fractional hypertree width, Marx
\cite{DBLP:journals/corr/abs-0911-0801}
recently introduced the concept of submodular width. This concept uses
a set of width functions satisfying a condition (submodularity), and
considers the $f$-width of a hypergraph for every such function $f$.

\begin{definition}[Submodular width function]
  Let $G = \tup{V, H}$ be a hypergraph. A width function $f$ on $G$ is
  \emph{edge-dominated} if $f(h) \leq 1$ for every $h \in H$.

  An edge-dominated width function $f$ on $G$ is \emph{submodular} if
  for every pair of sets $X, Y \subseteq V$, we have $f(X) + f(Y) \geq
  f(X \cap Y) + f(X \cup Y)$.
\end{definition}

\begin{definition}[Submodular width]
  Let $G$ be a hypergraph. The \emph{submodular width} $\subw(G)$ of
  $G$ is the supremum of the $f$-widths of $G$ taken over all
  monotone, edge-dominated, submodular width functions $f$ on $G$.

  For a class of hypergraphs $\mathcal H$, we write $\subw(\mathcal
  H)$ for the maximal submodular width over the hypergraphs in
  $\mathcal H$. If this is unbounded we write $\subw(\mathcal H) =
  \infty$; otherwise $\subw(\mathcal H) < \infty$.
\end{definition}

Unlike for fractional hypertree width and every other structural
restriction discussed so far, the running time of the algorithm given
by Marx for classic CSP instances with bounded submodular width has an
exponential dependence on the number of vertices in the hypergraph of
the instance. The class of classic CSP instances with bounded
submodular width is therefore not known to be tractable. However, this
class is what is called fixed-parameter tractable
\cite{downey2012parameterized,flum2006parameterized}.



\begin{definition}[Fixed-parameter tractable]
  A \emph{parameterized problem instance} is a pair $\tup{k, P}$,
  where $P$ is a problem instance, such as a CSP instance, and $k \in
  \mathbb N$ a parameter.

  Let $S$ be a class of parameterized problem instances. We say that
  $S$ is \emph{fixed-parameter tractable} (in \FPT) if there is a
  computable function $f$ of one argument, as well as a constant $c$,
  such that every problem $\tup{k, P} \in S$ can be solved in time
  $O(f(k) \times |P|^c)$.
\end{definition}

The function $f$ can be arbitrary, but must only depend on the
parameter $k$. For CSP instances, one possible parameterization is by
the size of the hypergraph of an instance, measured by the number of
vertices. Since the hypergraph of an instance has a vertex for every
variable, for every CSP instance $P = \tup{V, C}$ we consider the
parameterized instance $\tup{|V|, P}$.

\begin{theorem}[\cite{DBLP:journals/corr/abs-0911-0801}]
  \label{cor:subw}
  Let $\mathcal H$ be a class of hypergraphs. If $\subw(\mathcal H) <
  \infty$, then a class of classic CSP instances whose hypergraphs are
  in $\mathcal H$ is in \FPT{}.
\end{theorem}

The three structural restrictions that we have just presented form a
hierarchy
\cite{grohe-marx-frac-edge-cover,DBLP:journals/corr/abs-0911-0801}:
For every hypergraph $G$, $\subw(G) \leq \fhw(G) \leq \hw(G) \leq
\tw(G)$.

As the example below demonstrates, \cref{thm:width-tract} does not
hold for CSP instances with arbitrary global constraints, even if we
have a fixed, finite domain. The only exception is the restriction of
\cref{thm:width-tract} to treewidth, as bounded treewidth implies
bounded arity for every hyperedge.

\begin{example}
  \label{example:3col}
  The \NP-complete problem of 3-colourability
  \cite{Garey79:intractability} is to decide, given a graph
  $\tup{V,E}$, whether the vertices $V$ can be coloured with three
  colours such that no two adjacent vertices have the same colour.

  We may reduce this problem to a CSP with EGC constraints
  (cf.~Example~\ref{example:egc}) as follows: Let $V$ be the set of
  variables for our CSP instance, each with domain $\{r,g,b\}$. For
  every edge $\tup{v, w} \in E$, we post an EGC constraint with scope
  $\{v, w\}$, parameterized by the function $K$ such that $K(r) = K(g)
  = K(b) = \{0,1\}$. Finally, we make the hypergraph of this CSP
  instance have low width by adding an EGC constraint with scope $V$
  parameterized by the function $K'$ such that $K'(r) = K'(g) = K'(b)
  = \{0,\ldots,|V|\}$. This reduction clearly takes polynomial time,
  and the hypergraph $G$ of the resulting instance has
  $\hw(G)=\fhw(G)=\subw(G)=1$.

  As the constraint with scope $V$ allows
  all possible assignments, any solution to this CSP is also a
  solution to the 3-colourability problem, and vice versa.

\end{example}

Likewise, \cref{cor:subw} does not hold for CSP instances with
arbitrary global constraints if we allow the variables unbounded
domain size, that is, change the above example to allow each variable
its own set of colours. In other words, the structural restrictions
cannot yield tractable classes of CSP instances with arbitrary global
constraints. With that in mind, in the rest of the paper we will
identify properties of extensionally represented constraints that
these structural restrictions exploit to guarantee tractability. Then,
we are going to look for restricted classes of global constraints that
possess these properties. To do so, we will use the following
definitions.

\begin{definition}[Constraint \languageword]
  A \emph{constraint \languageword} is a set of global constraints. A
  CSP instance $\tup{V, C}$ is said to be over a constraint
  \languageword\ $\languagesymbol$ if for every $e[\delta] \in C$ we
  have $e[\delta] \in \languagesymbol$.
\end{definition}

\begin{definition}[Restricted CSP class]
\label{def:CSPrestricted}
Let $\languagesymbol$ be a constraint \languageword, and let $\mathcal
H$ be a class of hypergraphs.  We define $\CSP(\mathcal
H,\languagesymbol)$ to be the class of CSP instances over
$\languagesymbol$ whose hypergraphs are in $\mathcal H$.

\end{definition}

\Cref{def:CSPrestricted} allows us to discuss classic CSP instances
alongside instances with global constraints. Let $\Extlang$ be the
constraint \languageword{} containing all table global
constraints. The classic CSP instances are then precisely those that
are over $\Extlang$. In particular, we can now restate
\cref{thm:width-tract,cor:subw} as follows.

\begin{theorem}
  \label{thm:width-new-notation}
  Let $\mathcal H$ be a class of hypergraphs. For every $\alpha \in
  \{\tw, \hw,\fhw\}$, the class of CSP instances $\CSP(\mathcal H,
  \Extlang)$ is tractable if $\alpha(\mathcal H) <
  \infty$. Furthermore, if $\subw(\mathcal H) < \infty$ then
  $\CSP(\mathcal H, \Extlang)$ is in \FPT.
\end{theorem}

\section{Properties of Extensional Representation}
\label{sect:prop-ext}

We are going to start our investigation by considering fractional
hypertree width in more detail. To obtain tractability for classic CSP
instances of bounded fractional hypertree width, Grohe and Marx
\cite{grohe-marx-frac-edge-cover} use a bound on the number of
solutions to a classic CSP instance, and show that this bound is
preserved when we consider parts of a CSP instance. The following
definition formalizes what we mean by ``parts'', and is required to
state the algorithm that Grohe and Marx use in their paper.

\begin{definition}[Constraint projection]
  \label{def:constraint-projection}
  Let $e[\delta]$ be a global constraint. The \emph{projection of
    $e[\delta]$} onto a set of variables $X \subseteq \vars(\delta)$
  is the constraint $\Pj_X(e[\delta])$ such that $\mu \in
  \Pj_X(e[\delta])$ if and only if there exists $\theta \in e[\delta]$
  with $\theta|_X = \mu$.

  For a CSP instance $P = \tup{V, C}$ and $X \subseteq V$ we define
  $\Pj_X(P) = \tup{X, C'}$, where $C'$ is the set containing for every
  $e[\delta] \in C$ such that $X \cap \vars(\delta) \not= \emptyset$
  the constraint $\Pj_{X \cap \vars(\delta)}(e[\delta])$.
\end{definition}

\subsection{Algorithm for Enumerating All Solutions}

The algorithm is given as \cref{alg:enum-solutions}, and is
essentially the usual recursive search algorithm for finding all
solutions to a CSP instance by considering smaller and smaller
sub-instances using constraint projections.

\begin{algorithm}[h]
  \begin{algorithmic}
    \Procedure{EnumSolutions}{CSP instance $P = \tup{V, C}$}\Comment{Returns $\sol(P)$}
    \State{$\textup{Solutions} \leftarrow \emptyset$}
    \If {$V = \emptyset$}
      \State\Return{$\{\bot\}$}\Comment{The empty assignment}
    \Else
      \State{$w \leftarrow \textup{chooseVar}(V)$}\Comment{Pick a variable from $V$}
      \State{$\Theta = \textup{EnumSolutions}(\Pj_{V-\{w\}}(P))$}
      \For{$\theta \in \Theta$}
        \For{$a \in D(w)$}
          \State{$\theta'(w) = a$}
          \If{$\theta \oplus \theta'$ is a solution to $P$}
            \State $\textup{Solutions.add}(\theta \oplus \theta')$
          \EndIf
          \State{$\theta' \leftarrow \bot$}
        \EndFor
      \EndFor
    \EndIf
    \State\Return Solutions
    \EndProcedure
  \end{algorithmic}
  \caption{Enumerate all solutions of a CSP instance}
  \label{alg:enum-solutions}
\end{algorithm}

To show that \cref{alg:enum-solutions} does indeed find all solutions,
we will use the following property of constraint projections.

\begin{lemma}
  \label{lemma:solution-projection}
  Let $P = \tup{V, C}$ be a CSP instance. For every $X \subseteq V$,
  we have $\sol(\Pj_X(P)) \supseteq \pi_X(\sol(P))$.
\end{lemma}
\begin{proof}
  Given $P = \tup{V, C}$, let $X \subseteq V$ be arbitrary, and
  let $C' = \{e[\delta] \in C \mid X \cap \vars(\delta) \not=
  \emptyset\}$. For every $\theta \in \sol(P)$ and constraint
  $e[\delta] \in C'$ we have that $\theta|_{\vars(\delta)} \in
  e[\delta]$ since $\theta$ is a solution to $P$. By
  \cref{def:constraint-projection}, it follows that for every
  $e[\delta] \in C'$, $\theta|_{X \cap \vars(\delta)} \in \Pj_{X \cap
    \vars(\delta)}(e[\delta])$. Since the set of constraints of
  $\Pj_X(P)$ is the least set containing for each $e[\delta] \in C'$
  the constraint $\Pj_{X \cap \vars(\delta)}(e[\delta])$, we have
  $\theta|_X \in \sol(\Pj_X(P))$, and hence $\sol(\Pj_X(P)) \supseteq
  \pi_X(\sol(P))$. Since $X$ was arbitrary, the claim follows. \qed
\end{proof}


\begin{theorem}[Correctness of \cref{alg:enum-solutions}]
  For every CSP instance $P$, we have that $\textup{EnumSolutions}(P)
  = \sol(P)$.
\end{theorem}
\begin{proof}
  The proof is by induction on the set of variables $V$ in $P$. For
  the base case, if $V = \emptyset$, the empty assignment is the only
  solution.

  Otherwise, choose a variable $w \in V$, and let $X = V
  - \{w\}$. By induction, we can assume that
  $\textup{EnumSolutions}(\Pj_X(P)) = \sol(\Pj_X(P))$. Since for
  every $\theta \in \sol(P)$ there exists $a \in D(w)$ such that
  $\theta = \theta|_X \cup \tup{w, a}$, and furthermore $\theta|_X \in
  \pi_X(\sol(P))$, it follows by \cref{lemma:solution-projection} that
  $\theta|_X \in \sol(\Pj_X(P))$. Since \cref{alg:enum-solutions}
  checks every assignment of the form $\mu \cup \tup{w, a}$ for every
  $\mu \in \sol(\Pj_X(P))$ and $a \in D(w)$, it follows that
  $\textup{EnumSolutions}(P) = \sol(P)$. \qed
\end{proof}

The time required for this algorithm depends on three key factors,
which we are going to enumerate and discuss below. Let
\begin{enumerate}
\item $s(P)$ be the maximum of the number of solutions to each of the
  instances $\Pj_{W}(P)$, for $W \subseteq V$,
\item $c(P)$ be the maximum time required to check whether an
  assignment is a solution to $\Pj_{W}(P)$, for $W \subseteq V$, and
\item $b(P)$ be the maximum time required to construct any instance
  $\Pj_{W}(P)$, for $W \subseteq V$.
\end{enumerate}

There are $|V|$ calls to $\textup{EnumSolutions}$. For
each call, we need $b(P)$ time to construct the projection, while the
double loop takes at most $s(P) \times |D(w)| \times c(P)$
time. Therefore, letting $d = \max(\{|D(w)| \mid w \in V\})$,
the running time of \cref{alg:enum-solutions} is bounded by
$O\big(|V| \times (s(P) \times d \times c(P) + b(P))\big)$.

Since constructing the projection of a classic CSP instance can be
done in polynomial time, and likewise checking that an assignment is a
solution, the whole algorithm runs in polynomial time if $s(P)$ is a
polynomial in the size of $P$. For fractional edge covers, Grohe and
Marx show the following.

\begin{lemma}[\cite{grohe-marx-frac-edge-cover}]
  \label{lemma:fhw-few-solutions}
  A classic CSP instance $P$ has at most $|P|^{\rho^*(\hyp(P))}$
  solutions.
\end{lemma}

The reason for \cref{lemma:fhw-few-solutions} is that fractional edge
covers require the hypergraph to be quite dense, and also that the
hyperedges grow with the number of vertices in the hypergraph. This
result has since been shown to be optimal --- a classic CSP instance
has polynomially many solutions in its size if and only if it has
bounded fractional edge cover number
\cite{DBLP:journals/siamcomp/AtseriasGM13}.

Since fractional edge cover number is a monotone width function, it
follows that for any instance $P = \tup{V, C}$ and $X \subseteq V$,
$\rho^*(\hyp(\Pj_X(P))) \leq \rho^*(\hyp(P))$. This claim follows from
the fact that $\Pj_X(P)$ projects every constraint down to $X$, and
hence every hyperedge of $\hyp(P)$ down to $X$. Therefore, for classic
CSP instances of bounded fractional edge cover number $s(P)$ is indeed
polynomial in $|P|$. Grohe and Marx use this property to solve
instances with bounded fractional hypertree width (and hence, bounded
fractional edge cover number for every node in the corresponding tree
decomposition) in polynomial time.

\subsection{CSP Instances with Few Solutions in Key Places}
\label{sect:few-solutions}

As we have seen above, having few solutions for every projection of a
CSP instance is a property that can be used to obtain tractable
classes of classic CSP instances. More importantly, we have shown that
this property allows us to find all solutions to a CSP instance $P$,
even with global constraints, if we can build arbitrary projections of
$P$ in polynomial time. In other words, with these two conditions we
should be able to reduce instances with global constraints to classic
instances in polynomial time. This, in turn, should allows us to apply
the structural decomposition techniques discussed in
\cref{sect:structural-restrictions} to such instances.

However, on reflection there is no reason why we should need few
solutions for \emph{every} projection. Instead, consider the following
reduction.

\begin{definition}[Partial assignment checking]
  \label{def:part-assignment-checking}
  A global constraint \languageword{} $\Gamma$ allows \emph{partial
    assignment checking} if there exists a polynomial $p(n)$ such that
  for any constraint $e[\delta] \in \Gamma$ we can decide in time
  $O(p(|\delta|))$ whether a given assignment $\theta$ to a set of
  variables $W \subseteq \vars(\delta)$ is contained in an assignment
  that satisfies $e[\delta]$, i.e.~whether there exists $\mu \in
  e[\delta]$ such that $\theta = \mu|_W$.
\end{definition}

As an example, a \languageword{} that contains arbitrary EGC
constraints (cf.~\cref{example:egc}) does not satisfy
\cref{def:part-assignment-checking}, since checking whether an
arbitrary EGC constraint has a satisfying assignment is \NP-hard
\cite{quimper04-gcc-npc}. On the other hand, a \languageword{} that
contains only EGC constraints whose cardinality sets are intervals
does satisfy \cref{def:part-assignment-checking}
\cite{Regin96:aaai-generalized}.

If a \languageword{} $\Gamma$ satisfies
\cref{def:part-assignment-checking}, we can for any constraint
$e[\delta] \in \Gamma$ build arbitrary projections of it, that is,
construct the global constraint $\Pj_X(e[\delta])$ for any $X
\subseteq \vars(\delta)$, in polynomial time. In the case of
\cref{alg:enum-solutions}, where we build projections of projections,
we can do so by keeping a copy of the original constraint, and
projecting that each time.

\begin{definition}[Intersection variables]
  \label{def:intersection-vertices}
  Let $\tup{V, C}$ be a CSP instance. The set of \emph{intersection
    variables} of any constraint $e[\delta] \in P$ is $\iv(\delta) =
  \bigcup \{ \vars(\delta) \cap \vars(\delta') \mid e'[\delta'] \in C
  - \{e[\delta]\}\}$.
\end{definition}

Intersection variables are, in a sense, the only ``interesting''
variables of a constraint, as they are the ones interacting with the
rest of the problem.

\begin{definition}[Table constraint induced by a global constraint]
  \label{def:ind-const}
  Let $P = \tup{V, C}$ be a CSP instance. For every $e[\delta] \in C$,
  let $\mu^*$ be the assignment to $\vars(\delta) - \iv(\delta)$ that
  assigns a special value $*$ to every variable. The \emph{table
    constraint induced by $e[\delta]$} is $\ic(e[\delta]) =
  e'[\delta']$, where $\vars(\delta') = \vars(\delta)$, and $\delta'$
  contains for every assignment $\theta \in
  \sol(\Pj_{\iv(\delta)}(P))$ the assignment $\theta \oplus \mu^*$.
\end{definition}

If every constraint in a CSP instance $P = \tup{V, C}$ allows partial
assignment checking, then building $\ic(e[\delta])$ for any $e[\delta]
\in C$ can be done in polynomial time when $|\sol(\Pj_{X}(P))|$ is
itself polynomial in the size of $P$ for every subset $X$ of
$\iv(\delta)$. To do so, we can invoke \cref{alg:enum-solutions} on
the instance $\Pj_{\iv(\delta)}(P)$. The definition below expresses
this idea.

\begin{definition}[Sparse intersections]
  \label{def:sparse-intersections}
  A class of CSP instances $\mathcal P$ \emph{has sparse
    intersections} if there exists a constant $c$ such that for every
  constraint $e[\delta]$ in any instance $P \in \mathcal P$, we have
  that for every $X \subseteq \iv(\delta)$, $|\sol(\Pj_X(P))| <
  |P|^c$.
\end{definition}

If a class of instances $\mathcal P$ has sparse intersections, and the
instances are all over a constraint \languageword{} that allows
partial assignment checking, then we can for every constraint
$e[\delta]$ of any instance from $\mathcal P$ construct
$\ic(e[\delta])$ in polynomial time. While this definition considers
the instance as a whole, one special case of it is the case where
every constraint has few solutions in the size of its description,
that is, there is a constant $c$ and the constraints are drawn from a
\languageword{} $\Gamma$ such that for every $e[\delta] \in \Gamma$,
we have that $|\{\mu \mid \mu \in e[\delta]\}| < |\delta|^c$.

Note that the problem of checking whether a class of CSP instances
satisfies \cref{def:sparse-intersections} for a given $c$ is, in
general, hard. To see this, consider the special case of checking
whether a global constraint $e[\delta]$ has any satisfying assignments
at all. Letting $\delta$ be a SAT instance, that is, a propositional
formula, and $e$ an algorithm that checks whether an assignment to
$\vars(\delta)$ satisfies the formula makes this an \NP-hard problem
to solve. 

More generally, consider an arbitrary problem in \NP{}. By definition,
there is a polynomial-time algorithm that can check if a proposed
solution to such a problem is correct. By treating the algorithm as
the constraint type $e$, and the problem instances as descriptions
$\delta$, with a variable in $\vars(\delta)$ for each bit of the
solution, it becomes clear that every problem in \NP{} corresponds to
a class of global constraints. The fact that global constraints have
this much expressive power will be explored further in
\cref{sect:subproblem-decompositions}.

Despite such bad news, however, it is not always difficult to
recognise constraints with polynomially many satisfying assignments. A
trivial example would be table constraints. For a less trivial
example, consider the constraint $C^\beta$ from
\cref{example:cgp-as-csp}, where the number of satisfying assignments
is bounded by a polynomial with exponent $\beta$ (cf.~the discussion
after \cref{corollary:fewsol-csp} for a detailed analysis). 

For a more general example, consider a family of constraints that
satisfy \cref{def:part-assignment-checking}. To check whether the
number of solutions to a constraint from such a family is bounded by
$|\delta|^c$ for a fixed $c$ in polynomial time, we can use
\cref{alg:enum-solutions}, stopping it if the number of partial
assignments that extend to solutions exceeds the bound. Since we can
check whether a partial assignment extends to a solution in polynomial
time by \cref{def:part-assignment-checking}, we are also guaranteed an
answer in polynomial time.



Armed with these definitions, we can now state the following result.

\begin{theorem}
  \label{thm:decomp-to-classic}
  Let $\mathcal P$ be a class of CSP instances over a \languageword{}
  that allows partial assignment checking. If $\mathcal P$ has sparse
  intersections, then we can in polynomial time reduce any instance $P
  \in \mathcal P$ to a classic CSP instance $P_{CL}$ with $\hyp(P) =
  \hyp(P_{CL})$, such that $P_{CL}$ has a solution if and only if $P$
  does.
\end{theorem}
\begin{proof}
  Let $P=\tup{V, C}$ be an instance from such a class $\mathcal
  P$. For each $e[\delta] \in C$, $P_{CL}$ will contain the table
  constraint $\ic(e[\delta])$ from \cref{def:ind-const}. Since $P$ is
  over a \languageword{} that allows partial assignment checking, and
  $\mathcal P$ has sparse intersections, computing $\ic(e[\delta])$
  can be done in polynomial time by invoking \cref{alg:enum-solutions}
  on $\Pj_{\iv(\delta)}(P)$.

  By construction, $\hyp(P) = \hyp(P_{CL})$. All that is left to show
  is that $P_{CL}$ has a solution if and only if $P$ does. Let
  $\theta$ be a solution to $P = \tup{V, C}$. For every $e[\delta] \in
  C$, we have that $\theta|_{\iv(\delta)} \in \Pj_{\iv(\delta)}(P)$ by
  \cref{def:constraint-projection,def:intersection-vertices}, and the
  assignment $\mu$ that assigns the value $\theta(v)$ to each $v \in
  \displaystyle\bigcup_{e[\delta] \in C}\iv(\delta)$, and $*$ to every
  other variable is therefore a solution to $P_{CL}$.

  In the other direction, if $\theta$ is a solution to $P_{CL}$, then
  $\theta$ satisfies $\ic(e[\delta])$ for every $e[\delta] \in C$. By
  \cref{def:ind-const}, this means that $\theta|_{\iv(\delta)} \in
  \sol(\Pj_{\iv(\delta)}(P))$, and by
  \cref{def:constraint-projection}, there exists an assignment
  $\mu^{e[\delta]}$ with $\mu^{e[\delta]}|_{\iv(\delta)} =
  \theta|_{\iv(\delta)}$ that satisfies $e[\delta]$. By
  \cref{def:intersection-vertices}, the variables not in $\iv(\delta)$
  do not occur in any other constraint in $P$, so we can combine all
  the assignments $\mu^{e[\delta]}$ to form a solution $\mu$ to $P$
  such that for $e[\delta] \in C$ and $v \in \vars(\delta)$ we have
  $\mu(v) = \mu^{e[\delta]}(v)$. \qed
\end{proof}

From \cref{thm:decomp-to-classic}, we get tractable and
fixed-parameter tractable classes of CSP instances with global
constraints, in particular by applying \cref{thm:width-new-notation}.

\begin{corollary}
  \label{corollary:fewsol-csp}
  Let $\mathcal H$ be a class of hypergraphs, and $\languagesymbol$ a
  \languageword{} that allows partial assignment checking. If
  $\CSP(\mathcal H, \languagesymbol)$ has sparse intersections, then
  $\CSP(\mathcal H, \languagesymbol)$ is trac\-ta\-ble or in \FPT{} if
  $\CSP(\mathcal H, \Extlang)$ is.
\end{corollary}
\begin{proof}
  Let $\mathcal H$ and $\languagesymbol$ be given. By
  \cref{thm:decomp-to-classic}, we can reduce any $P \in \CSP(\mathcal
  H, \languagesymbol)$ to an instance $P_{CL} \in \CSP(\mathcal H,
  \Extlang)$ in polynomial time. Since $P_{CL}$ has a solution if and
  only if $P$ does, tractability or fixed-parameter tractability of
  $\CSP(\mathcal H, \Extlang)$ implies the same for $\CSP(\mathcal
  H,\languagesymbol)$. \qed
\end{proof}

To illustrate the above result, consider again the connected graph
partition problem (\cref{prob:cgp}). This problem is \NP-complete
\cite[p.~209]{Garey79:intractability}, even for fixed $\alpha \geq
3$. However, note that when $\beta$ is fixed, we can solve the problem
in polynomial time, by successively guessing sets $E'$, with $|E'|\leq
\beta$, of broken edges, and checking whether the connected components
of the graph $\tup{V, E - E'}$ all have $\alpha$ or fewer
vertices. The number of such sets $E'$ is bounded by
$\displaystyle\sum_{i=1}^{\beta} \binom{|E|}{i} \leq (|E|+1)^\beta$,
which is polynomial if $\beta$ is fixed. As we show below, this
argument can be seen as a special case of
\cref{thm:decomp-to-classic}. To simplify the analysis, we assume
without loss of generality that $\alpha < |V|$, which means that any
solution has at least one broken edge.

We claim that if $\beta$ is fixed, then the constraint $C^\beta =
e^\beta[\delta^\beta]$ allows partial assignment checking, and has
only a polynomial number of satisfying assignments. The latter implies
that for any instance $P$ of the CGP,
$|\sol(\Pj_{\iv(\delta^\beta)}(P))|$ is polynomial in the size of $P$
for every subset of $\iv(\delta^\beta)$. Furthermore, we will show
that for the constraint $C^\alpha = e^\alpha[\delta^\alpha]$, we also
have that $|\sol(\Pj_{\iv(\delta^\alpha)}(P))|$ is polynomial in the
size of $P$. That $C^\alpha$ allows partial assignment checking can be
seen by noting that each variable in $\vars(\delta^\alpha)$ has a
domain value for every vertex in the underlying graph. Therefore,
given a partial assignment to $\vars(\delta^\alpha)$, we can check
that no value is assigned more than $\alpha$ times. If yes, this
assignment can be extended to a full one by assigning each remaining
variable a domain value not yet assigned to any variable.

First, we show that the number of satisfying assignments to $C^\beta$
is limited. Since $C^\beta$ limits the number of ones in any solution
to $\beta$, the number of satisfying assignments to this constraint is
the number of ways to choose up to $\beta$ variables to be assigned
one. This is bounded by $\displaystyle\sum^{\beta}_{i=1}\binom{|E|}{i}
\leq (|E|+1)^\beta$, and so we can generate them all in polynomial
time. This argument also implies that we can perform partial
assignment checking, simply by looking at the generated assignments.

Now, let $\theta$ be such a solution. How many solutions to $P$
contain $\theta$? Every constraint on $\{u, v, e\}$ with $\theta(e)=1$
allows at most $|V|^2$ assignments, and there are at most $\beta$ such
constraints. So far we therefore have at most $(|E|+1)^\beta \times
|V|^{2\beta}$ assignments.

On the other hand, a ternary constraint with $\theta(e)=0$ requires
$\theta(u) = \theta(v)$. Consider the graph $G_0$ containing for every
constraint on $\{u, v, e\}$ with $\theta(e)=0$ the vertices $u$ and
$v$ as well as the edge $\{u, v\}$. Since the original graph was
connected, every connected component of $G_0$ contains at least one
vertex which is in the scope of some constraint with
$\theta(e)=1$. Therefore, since equality is transitive, each connected
component of $G_0$ allows at most one assignment for each of the
$(|E|+1)^\beta \times |V|^{2\beta}$ assignments to the other variables
of $P$. We therefore get a total bound of $(|E|+1)^\beta \times
|V|^{2\beta}$ on the total number of solutions to $P$, and hence to
$\Pj_{\iv(\delta^\alpha)}(P)$.

The hypergraph of any CSP instance $P$ encoding the CGP has two
hyperedges covering the whole problem, so the hypertree width of this
hypergraph is two. Therefore,
\cref{corollary:fewsol-csp,thm:width-tract} apply and yield
tractability for fixed $\beta$.

\subsection{Back Doors}

\label{sect:back-doors}

If a class of CSP instances includes constraints from a
\languageword{} that is not known to allow partial assignment
checking, we may still obtain tractability in some cases by applying
the notion of a back door set. A (strong) back door set
\cite{gaspers-backdoors,Williams03backdoorsto} is a set of variables
in a CSP instance that, when assigned, make the instance easy to
solve. Below, we are going to adapt this notion to individual
constraints.

\begin{definition}[Back door]
  \label{def:back-doors}
  Let $\languagesymbol$ be a global constraint \languageword{}. A
  \emph{back door} for a constraint $e[\delta] \in \languagesymbol$ is
  any set of variables $W \subseteq \vars(\delta)$ (called a back door
  set) such that we can decide in polynomial time whether a given
  assignment $\theta$ to a set of variables $\vars(\theta) \supseteq
  W$ is contained in an assignment that satisfies $e[\delta]$,
  i.e.~whether there exists $\mu \in e[\delta]$ such that
  $\mu|_{\vars(\theta)} = \theta$.
\end{definition}

Trivially, for every constraint $e[\delta]$ the set of variables
$\vars(\delta)$ is a back door set, since by \cref{def:glob-const} we
can always check in polynomial time if an assignment to
$\vars(\delta)$ satisfies the constraint $e[\delta]$.

The key point about back doors is that given a \languageword{}
$\Gamma$, adding to each $e[\delta] \in \Gamma$ with back door set $W$
an arbitrary set of assignments to $W$ produces a \languageword{}
$\Gamma'$ that allows partial assignment checking. Adding a set of
assignments $\Theta$ means to add $\Theta$ to the description, and
modify the algorithm $e$ to only accept an assignment if it contains a
member of $\Theta$ in addition to previous requirements. Furthermore,
given a CSP instance $P$ containing $e[\delta]$, as long as $\Theta
\supseteq \pi_{W}(\sol(P))$, adding $\Theta$ to $e[\delta]$ produces
an instance that has exactly the same solutions. This point leads to
the following definition.

\begin{definition}[Sparse back door cover]
  \label{def:sparse-back-door-cover}
  Let $\Gamma_{PAC}$ be a \languageword{} that allows partial
  assignment checking and $\Gamma_{BD}$ a \languageword{}. For every
  instance $P = \tup{V, C}$ over $\Gamma_{PAC} \cup \Gamma_{BD}$, let
  $P \cap \Gamma_{PAC}$ be the instance with constraint set $C' = C
  \cap \Gamma_{PAC}$ and set of variables $\bigcup \{V \cap
  \vars(\delta) \mid e[\delta] \in C'\}$.

  A class of CSP instances $\mathcal P$ over $\Gamma_{PAC} \cup
  \Gamma_{BD}$ has \emph{sparse back door cover} if there exists a
  constant $c$ such that for every instance $P = \tup{V, C} \in
  \mathcal P$ and constraint $e[\delta] \in C$, if $e[\delta] \not\in
  \Gamma_{PAC}$, then there exists a back door set $W$ for
  $e[\delta]$, findable in time polynomial in $|P|$, such that
  $|\sol(\Pj_X(P \cap \Gamma_{PAC}))| \leq |P|^c$ for every $X
  \subseteq W$.
\end{definition}

Sparse back door cover means that for each constraint that is not from
a \languageword{} that allows partial assignment checking, we can in
polynomial time get a set of assignments $\Theta$ for its back door
set using \cref{alg:enum-solutions}, and so turn this constraint into
one that does allow partial assignment checking. This operation
preserves the solutions of the instance that contains this constraint.

\begin{theorem}
  \label{lemma:backdoors}
  If a class of CSP instance $\mathcal P$ has sparse back door cover,
  then we can in polynomial time reduce any instance $P \in \mathcal
  P$ to an instance $P'$ such that $\hyp(P) = \hyp(P')$ and $\sol(P) =
  \sol(P')$. Furthermore, the class of instances $\{P' \mid P \in
  \mathcal P\}$ is over a \languageword{} that allows partial
  assignment checking.
\end{theorem}
\begin{proof}
  Let $P =\tup{V, C} \in \mathcal P$. We construct $P'$ by adding to
  every $e[\delta] \in C$ such that $e[\delta] \not\in \Gamma_{PAC}$,
  with back door set $W$, the set of assignments $\sol(\Pj_W(P \cap
  \Gamma_{PAC}))$, which we can obtain using
  \cref{alg:enum-solutions}. By \cref{def:sparse-back-door-cover}, we
  have for every $X \subseteq W$ that $|\sol(\Pj_W(P \cap
  \Gamma_{PAC}))| \leq |P|^c$, so \cref{alg:enum-solutions} takes
  polynomial time since $\Gamma_{PAC}$ does allow partial assignment
  checking.

  It is clear that $\hyp(P') = \hyp(P)$, and since $\sol(\Pj_W(P \cap
  \Gamma_{PAC})) \supseteq \pi_W(\sol(P))$, the set of solutions stays
  the same, i.e.~$\sol(P') = \sol(P)$. Finally, since we have replaced
  each constraint $e[\delta]$ in $P$ that was not in $\Gamma_{PAC}$ by
  a constraint that does allow partial assignment checking, it follows
  that $P'$ is over a \languageword{} that allows partial assignment
  checking. \qed
\end{proof}

One consequence of \cref{lemma:backdoors} is that we can sometimes
apply \cref{thm:decomp-to-classic} to a CSP instance that contains a
constraint for which checking if a partial assignment can be extended
to a satisfying one is hard. We can do so when the variables of that
constraint are covered by the variables of other constraints that do
allow partial assignment checking --- but only if the instance given
by those constraints has few solutions.

As a concrete example of this, consider again the encoding of the CGP
that we gave in \cref{example:cgp-as-csp}. The variables of constraint
$C^\alpha$ are entirely covered by the instance $P'$ obtained by
removing $C^\alpha$. As the entire set of variables of a constraint is
a back door set for it, and the instance $P'$ has few solutions
(cf.~the discussion after \cref{thm:decomp-to-classic}), this class of
instances has sparse back door cover. As such, the constraint
$C^\alpha$ could, in fact, be arbitrary without affecting the
tractability of this problem. In particular, the requirement that
$C^\alpha$ allows partial assignment checking can be dropped.

\section{Subproblem Decompositions}
\label{sect:subproblem-decompositions}

To generalize \cref{thm:decomp-to-classic}, consider the fact that our
definition of a global constraint allows us to view a CSP instance
$\tup{V, C}$ as a single constraint $e[\delta]$, by letting $\delta$
contain the set of constraint $C$, and setting $\vars(\delta) =
V$. The algorithm $e$ then checks if an assignment satisfies all
constraints. Of course, such a constraint encodes an \NP-complete
problem, but this is no different from e.g.~the EGC constraint
\cite{quimper04-gcc-npc} (cf.~\cref{example:egc}). With this in mind,
in this section we are going to investigate what happens if a CSP
instance is split up into a set of smaller instances.

Splitting up a (classic) CSP instance into smaller instances has
previously been considered by Cohen and Green
\cite{guarded-decomp-CG}. They use a very general framework of guarded
decompositions \cite{CohenUnifTract} to define what they call ``typed
guarded decompositions''. This notion allows them to obtain a
tractability result for a CSP instance that can be split into smaller
instances drawn from known tractable classes.

In this section, we are going to adapt the notions defined in
\cref{sect:few-solutions} to work with CSP instances rather than
single constraints. Then, in \cref{sect:cohen-green}, we will show how
the result of Cohen and Green can be derived as a special case of
\cref{corollary:subproblem-csp}.

\begin{definition}[CSP subproblem]
  Given two CSP instances $P = \tup{V, C}$ and $P' = \tup{V', C'}$, we
  say that $P'$ is a \emph{subproblem} of $P$ if $C' \subseteq C$.
\end{definition}

In other words, a subproblem of a CSP instance is given by a subset of
the constraints in that instance. In \cite{guarded-decomp-CG}, Cohen
and Green call a subproblem a \emph{component} of $P$.

\begin{definition}[CSP union]
  Let $Q_1 = \tup{V_1, C_1}$ and $Q_2 = \tup{V_2, C_2}$ be two CSP
  instances. The \emph{union} of $Q_1$ and $Q_2$ is the instance $Q_1
  \sqcup Q_2 = \tup{V_1 \cup V_2, C_1 \cup C_2}$.

\end{definition}


\begin{definition}[Subproblem decomposition]
  \label{def:subproblem-decomp}
  Let $P$ be a CSP instance. A set $S$ of subproblems of $P$ is a
  \emph{subproblem decomposition} of $P$ if $\bigsqcup S = P$.

  A subproblem decomposition of a CSP instance is \emph{proper} if no
  element of the decomposition is a subproblem of any other.
\end{definition}

A subproblem decomposition of an instance $P$, then, is a set of
subproblems that together contain all the constraints and variables of
$P$. Note that a constraint may occur in more than one subproblem in a
decomposition.

Below, we shall assume that all subproblem decompositions are
proper. Since subproblems are given by subsets of constraints, the
solutions to a CSP instance can be turned into solutions for any
subproblem by projecting out the variables not part of the
subproblem. Therefore, solving a subproblem $P$ that contains another
subproblem $P'$ also solves $P'$, making $P'$ redundant.

\begin{example}
  Let $P = \tup{V, C}$ be a CSP instance. A very simple subproblem
  decomposition of $P$ would be $\{\tup{\vars(\delta), e[\delta]} \mid
  e[\delta] \in C\}$, that is, every constraint of $P$ is a separate
  subproblem. This subproblem decomposition is clearly proper.
\end{example}

\begin{example}
  \label{example:subproblem-decompositions-simple}
  Consider a family of CSP instances on the set of boolean variables
  $\{x_i, y_i, z_{i} \mid 1 \leq i \leq n \in \{4, 6, 8, \ldots\}\}$,
  with the following constraints: An EGC constraint $A$ on
  $\{x_1,\ldots,x_n\}$ with $K(1)=4$ and $K(0) = \{0,\ldots,n\}$. A
  second EGC constraint $B$, on $\{y_1,\ldots,y_{n},z_1,\ldots,z_n\}$
  with $K(1) = K(0) = \{n\}$, and binary constraints on each pair
  $\{x_i, y_i\}$ enforcing equality. A possible subproblem
  decomposition for an instance from this family would be $\{P, Q\}$,
  where $P$ contains $A$ as well as the binary constraints, and $Q$
  contains the constraint $B$. This family is depicted in
  \cref{fig:subproblem-decomp-simple}, with $P$ containing the
  constraints marked by solid lines, and $Q$ the constraint marked by
  a dashed line.
\end{example}

\begin{figure}[h]
  \centering
  \includegraphics[width=0.8\textwidth]{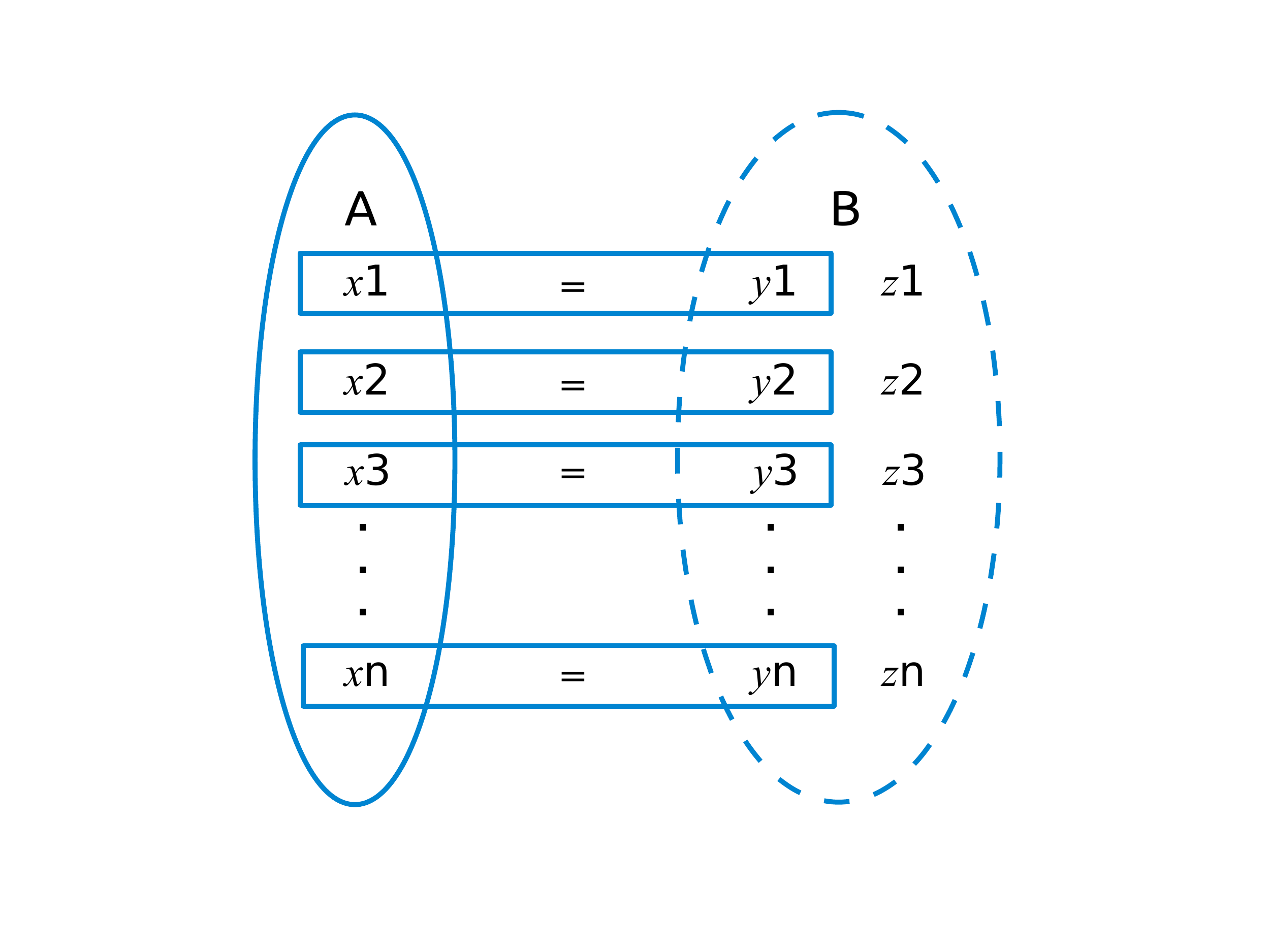}
  \caption{Family of instances from
    \cref{example:subproblem-decompositions-simple} with decomposition
    $\{P, Q\}$. Subproblem $P$ marked with solid lines and $Q$ with a
    dashed line.}
  \label{fig:subproblem-decomp-simple}
\end{figure}

Viewing subproblems as constraints and a subproblem decomposition $S$
as a CSP instance $\tup{\vars(\bigsqcup S), S}$, we have
$\sol(\tup{\vars(\bigsqcup S), S}) = \sol(\bigsqcup S)$, since every
constraint is in some subproblem. As such, we will treat $S$ as a CSP
instance when it is convenient to simplify notation.

Using \cref{def:subproblem-decomp}, we can treat any set of
CSP instances $S$ as a subproblem decomposition of the instance
$\bigsqcup S$. With that in mind, whenever we say that $S$ is a
subproblem decomposition without specifying what it is a decomposition
of, we mean that $S$ is a decomposition of the CSP instance $\bigsqcup
S$.

\begin{definition}[CSP instances given by subproblem decompositions]
  \label{def:class-given-by-subproblems}
  Let $\mathcal F$ be a family of subproblem decompositions. We define
  $\CSP(\mathcal F)$ to be the class of CSP instances $\{ \bigsqcup S
  \mid S \in \mathcal F \}$.
\end{definition}

\begin{definition}[Hypergraph of a subproblem decomposition]
  \label{def:hyp-of-subproblem-decomp}
  Let $S$ be a subproblem decomposition. The hypergraph of $S$,
  denoted $\hyp(S)$, has vertex set $\vars(\bigsqcup S)$ and set of
  hyperedges $\{ \vars(P) \mid P \in S \}$.

  For a family $\mathcal F$ of subproblem decompositions, let
  $\hyp(\mathcal F) = \{\hyp(S) \mid S \in \mathcal F\}$.
\end{definition}


Since a CSP instance can be seen as a global constraint,
\cref{def:part-assignment-checking} (partial assignment checking) and
\cref{def:sparse-intersections} (sparse intersections) carry over
unchanged. To apply them to a family of subproblem decompositions
$\mathcal F$, we need only consider the \languageword{} $\bigcup
\mathcal F$ in both cases.

One way of interpreting \cref{def:part-assignment-checking} for a
\languageword{} of CSP instances is that every instance has been drawn
from a tractable class --- not necessarily the same one, as long as
these classes all allow us to check in polynomial time whether a
partial assignment extends to a solution. Most known tractable classes
of CSP instances have this property; in particular, all the classes
discussed in \cref{sect:structural-restrictions} have it. To see this,
note that a partial assignment can be seen as a set of constraints on
one variable each, and adding such hyperedges to a hypergraph does not
change its tree, hypertree, or submodular width. On the other hand,
tractable classes defined by restricting the allowed assignments of a
constraint, rather than the hypergraph, are usually preserved by
adding a constraint with only one assignment
\cite{Cohen-comp-lang-handbook}.

To illustrate how these definitions apply to subproblem
decompositions, consider the following example.

\begin{example}
  \label{example:sparse-intersections}
  Recall the family of subproblem decompositions in
  \cref{example:subproblem-decompositions-simple}. For a decomposition
  $S = \{P, Q\}$ from this family, the set of intersection vertices
  for both subproblems is $\{y_1,\ldots,y_n\}$. Furthermore, the EGC
  constraint $A$ requires that there are exactly $4$ variables
  assigned $1$ among $\{x_1,\ldots,x_n\}$, so there are $\binom{n}{4}$
  satisfying assignments for this constraint. The equality constraints
  ensure that this is the number of solutions to the whole subproblem
  $P$, so for every $X \subseteq \{y_1,\ldots,y_n\}$ we have that
  $|\sol(\Pj_X(S))| \leq \binom{n}{4}$. Therefore, this family of
  subproblem decompositions has sparse intersections.
\end{example}

We can now derive a straightforward generalization of
\cref{thm:decomp-to-classic}.

\begin{theorem}
  \label{thm:subp-decomp-to-classic}
  Let $\mathcal F$ be a family of subproblem decompositions that
  allows partial assignment checking. If $\mathcal F$ has sparse
  intersections, then we can in polynomial time reduce any subproblem
  decomposition $S \in \mathcal F$ to a classic CSP instance $P$ with
  $\hyp(P) = \hyp(S)$, such that $P$ has a solution if and only if $S$
  does.
\end{theorem}
\begin{proof}
  As subproblems can be seen as global constraints, the proof follows
  directly from \cref{thm:decomp-to-classic}. \qed
\end{proof}

\begin{corollary}
  \label{corollary:subproblem-csp}
  Let $\mathcal F$ be a family of subproblem decompositions that
  allows partial assignment checking and has sparse intersections. If
  $\CSP(\hyp(\mathcal F), \Extlang)$ is trac\-ta\-ble or in \FPT{},
  then so is $\CSP(\mathcal F)$.
\end{corollary}
\begin{proof}
  Let $\mathcal F$ be given. By \cref{thm:subp-decomp-to-classic}, we
  can reduce any subproblem decomposition $S \in \mathcal F$ to an
  instance $P \in \CSP(\hyp(\mathcal F), \Extlang)$ in polynomial
  time. Since $P$ has a solution if and only if $S$ does, tractability
  of $\CSP(\hyp(\mathcal F), \Extlang)$ implies the same for
  $\CSP(\mathcal F)$. \qed
\end{proof}

To illustrate this result, recall
\cref{example:subproblem-decompositions-simple}. From
\cref{example:sparse-intersections}, we know that this family of
subproblem decompositions has sparse intersections. Furthermore, both
subproblem allow partial assignment checking, as the EGC constraints
both have interval cardinality sets \cite{Regin96:aaai-generalized},
and the equality constraints of subproblem $P$ can always be
satisfied. Therefore, \cref{corollary:subproblem-csp} applies to this
problem.

\subsection{Applying \cref{corollary:subproblem-csp}}
\label{sect:cohen-green}

We are now ready to discuss the result of Cohen and Green mentioned at
the beginning of \cref{sect:subproblem-decompositions}, and to show
how it can be derived as a special case of our result. First, we need
to define guarded decompositions.

\begin{definition}[Guarded decomposition]
  A guarded block of a hypergraph $G$ is a pair $\tup{\lambda, \chi}$
  where the guard $\lambda$ is a subset of the hyperedges of $G$, and
  the block, $\chi$, is a subset of $\bigcup \lambda$.

  For every classic CSP instance $P$ and every guarded block
  $\tup{\lambda, \chi}$ of $\hyp(P)$, we define the constraint
  generated by $P$ on $\tup{\lambda, \chi}$ to be the projection onto
  $\chi$ of the join of all the constraints of $P$ whose scopes are in
  $\lambda$.

  A set of of guarded blocks $\Theta$ of a hypergraph $G$ is a guarded
  decomposition of $G$ if for every $P \in \CSP(\{G\}, \Extlang)$, the
  CSP instance over the same variables as $P$ with constraints
  generated by the blocks in $\Theta$ has the same solutions as $P$. 

  A guarded decomposition is acyclic if the hypergraph having the
  union of the blocks $\chi$ as vertices, and each $\chi$ as a
  hyperedge, is acyclic.
\end{definition}

Cohen and Green then introduce a mapping $\mu$ from the constraints of
a CSP instance $P$ to nonempty sets of elements of a guarded
decomposition of $\hyp(P)$. They demand that
\begin{enumerate}
\item For each guarded block $\tup{\lambda, \chi}$ and hyperedge in
  $\lambda$, $\mu$ assigns at least one constraint with that scope to
  this guarded block,
\item that the set of guarded blocks $\mu$
assigns to a constraint $c$ contains the scope of $c$ in all the
guards, and finally
\item that at least one of the guarded blocks assigned
to $c$ contains the variables of the scope of $c$ in the block.
\end{enumerate}

Note that, taken together, the conditions above mean that the mapping
$\mu$ turns each guarded block of the decomposition into a subproblem,
and the whole decomposition into a subproblem decomposition, since
each guarded block is assigned a set of constraints, and each
constraint is assigned to a guarded block.

Furthermore, they introduce two more notions. A type is a
polynomial-time algorithm for solving a set of CSP instances. A typed
guarded decomposition is one where each guarded block $\beta$ is
assigned a type, and the CSP instance given by the set of constraints
assigned to $\beta$ is a member of the assigned type. This is almost
\cref{def:part-assignment-checking}, however, there is no provision
for solving a problem with some variables assigned.

Finally, a guarded decomposition $\Theta$ is $k$-separated if for
every guarded block $\tup{\lambda, \chi}$ there exists a set of
hyperedges $\epsilon$, with $|\epsilon| \leq k$, such that for each
guarded block $\tup{\lambda_2, \chi_2} \in \Theta - \{\lambda, \chi\}$
we have that $\chi \cap \chi_2 \subseteq \bigcup \epsilon$. Observe
that when $k$ is fixed, the intersection variables of each subproblem
are covered by a fixed number of table constraints, and hence that the
number of possible solutions is bounded by the size of the join of
these constraints. It follows that the intersections are sparse as per
\cref{def:sparse-intersections}.

They then proceed to show that for fixed $k$, a CSP instance with a
$k$-separated, acyclic typed guarded decomposition can be solved in
polynomial time, \emph{under the condition that the types can handle
  problems with some variables assigned specific values}.

The last condition is precisely what we need for partial assignment
checking. Therefore, since the decomposition is required to be
acyclic, their result satisfies the conditions of
\cref{corollary:subproblem-csp}. Note, however, that since there are
other ways to obtain sparse intersections,
\cref{corollary:subproblem-csp} is a more general result even for
classic CSP instances.

\section{Weighted CSP}

Having few solutions in key parts of a CSP instance has turned out to
be a property we can exploit to obtain tractability. In this section,
we are going to apply this property to an extension of the CSP
framework called weighted CSP instances
\cite{gottlob-etal-tractable-optimization,degivry06}, where every
constraint assigns a cost to every satisfying assignment, and we would
like to find a solution with smallest cost. This type of CSP is itself
a special case of the more general valued CSP framework
\cite{schiex95vcsp,standa-vcsp}, where every constraint is specified
by a function that assigns a cost to every possible assignment for the
variables of that constraint. The reason for considering weighted,
rather than valued, CSP, is that weighted (table) constraints list
every satisfying assignment along with the costs, while a valued
constraint is given by a function from assignments to values. The
representation of a valued constraint is thus much more compact, and
the notion of a satisfying assignment is no longer defined.

\begin{definition}[Weighted constraint]
  A \emph{weighted global constraint} $e[\delta]$ is a global
  constraint that assigns to each $\theta \in e[\delta]$ a value
  $\cost(e[\delta], \theta)$ from $\mathbb Q$.

  The \emph{size} of a weighted global constraint $e[\delta]$ is given
  by the sum of $|\delta|$ and the size of the bit representation for
  each cost.
\end{definition}

In other words, the number of bits needed to represent the costs of
all the satisfying assignments is part of a weighted constraint's
size.

\begin{definition}[WCSP instance]
  \label{def:wcsp}
  A \emph{WCSP instance} is a pair $P = \tup{V, C}$, where $V$ is a
  set of variables and $C$ a set of weighted constraints. An
  assignment is a solution to $P$ if it satisfies every constraint in
  $C$, and we denote the set of all solutions to $P$ by $\sol(P)$.

  For every solution $\theta$ to $P$ we define $\cost(P, \theta) =
  \displaystyle\sum_{e[\delta] \in C} \cost(e[\delta],
  \theta|_{\vars(\delta)})$. An assignment $\theta$ is an
  \emph{optimal} solution to $P$ if and only if it is a solution to
  $P$ \emph{with the smallest cost}, i.e.~$\cost(P, \theta) =
  \min(\{\cost(P, \theta') \mid \theta' \in \sol(P)\})$.
\end{definition}

As is commonly done with optimization problems in complexity theory,
below we consider the decision problem associated with WCSP instances.

\begin{definition}[WCSP decision problem]
  Given a WCSP instance $P$ and $k \in \mathbb Q$, the \emph{WCSP
    decision problem} is to decide whether $P$ has a solution $\theta$
  with $\cost(P, \theta) \leq k$.
\end{definition}

As for CSP instances, a classic WCSP instance is one where all
constraints are table global constraints. As an example of known
tractability results for classic WCSP instances, consider the theorem
below.

\begin{theorem}[\cite{gottlob-etal-tractable-optimization}]
  Let $\mathcal H$ be a class of hypergraphs. If $\hw(\mathcal H) <
  \infty$, then a class of classic WCSP instances whose hypergraphs
  are in $\mathcal H$ is tractable.
\end{theorem}

Since we are free to ignore the costs a weighted constraint puts on
assignments and treat it as an ``ordinary'' constraint, definitions of
subproblems and subproblem decompositions carry over unchanged. Note
that since the WCSP decision problem is clearly in \NP, we can view a
WCSP instance as a \emph{weighted} global constraint. Therefore,
\cref{def:part-assignment-checking} will now be subtly different.

\begin{definition}[Weighted part. assignment checking]
  \label{def:weighted-part-assignment-checking}
  A weighted con\-straint \languageword{} $\Gamma$ allows
  \emph{partial assignment checking} if for any weighted constraint
  $e[\delta] \in \Gamma$ we can decide in polynomial time, given an
  assignment $\theta$ to a set of variables $W \subseteq
  \vars(\delta)$ and $k \in \mathbb Q$, whether $\theta$ is contained
  in an assignment that satisfies $e[\delta]$ and has cost at most
  $k$, i.e.~whether there exists $\mu \in e[\delta]$ such that $\theta
  = \mu|_W$ and $\cost(e[\delta], \mu) \leq k$.
\end{definition}

In other words, given a partial assignment we need to be able to solve
the WCSP decision problem for our constraint in polynomial time. Note
also that doing so allows us to find the minimum cost among the
assignments that contain our partial assignment by binary search. This
will be needed in order to construct projections of a weighted global
constraint. To define the projection of a weighted constraint, we need
to alter \cref{def:constraint-projection} to take costs into account.

\begin{definition}[Weighted constraint projection]
  \label{def:weighted-constraint-projection}
  Let $e[\delta]$ be a weighted constraint. The \emph{projection} of
  $e[\delta]$ onto a set of variables $X \subseteq \vars(\delta)$ is
  the constraint $\Pj_X(e[\delta])$ such that $\mu \in
  \Pj_X(e[\delta])$ if and only if there exists $\theta \in e[\delta]$
  with $\theta|_X = \mu$. The cost of an assignment $\theta \in
  \Pj_X(e[\delta])$ is $\cost(\Pj_X(e[\delta]), \theta) =
  \min(\{\cost(e[\delta], \mu) \mid \mu \in e[\delta] \mbox{ and }
  \mu|_X = \theta\})$.

  For a WCSP instance $P = \tup{V, C}$ and $X \subseteq V$ we define
  $\Pj_X(P) = \tup{X, C'}$, where $C'$ is the least set containing for
  every $e[\delta] \in C$ such that $X \cap \vars(\delta) \not=
  \emptyset$ the constraint $\Pj_{X \cap \vars(\delta)}(e[\delta])$.
\end{definition}

\begin{definition}[Weighted table constraint induced by a subproblem]
  \label{def:weighted-ind-const}
  Let $S$ be a subproblem decomposition. For every $T \in S$, let
  $\mu^*$ be the assignment to $\vars(T) - \iv(T)$ that assigns a
  special value $*$ to every variable. The \emph{weighted table
    constraint induced by $T$} is $\ic(T) = e[\delta]$, where
  $\vars(\delta) = \vars(T)$, and $\delta$ contains for every
  assignment $\theta \in \sol(\Pj_{\iv(T)}(S))$ the assignment $\theta
  \oplus \mu^*$ with $\cost(\ic(T), \theta \oplus \mu^*) =
  \cost(\Pj_{\iv(T)}(T),\theta)$.
\end{definition}

Since the variables of a subproblem $T \in S$ not in $\iv(T)$ occur
only in $T$ itself, if we have a solution to $\Pj_{\iv(T)}(S)$, it
doesn't matter what solution to $T$ we extend it to. We should
therefore pick the one that has the smallest cost, and that cost is
precisely $\cost(\Pj_{\iv(T)}(T),\theta)$ by
\cref{def:weighted-constraint-projection}. The same as for CSP
instances, if every subproblem in a weighted decomposition $S$ allows
weighted partial assignment checking, building $\ic(T)$ for any $T \in
S$ can be done in polynomial time when $|\sol(\Pj_{\iv(T)}(S))|$ is
polynomial in the size of $\bigsqcup S$ for every subset of $\iv(T)$,
again by using \cref{alg:enum-solutions}. Since the definition of
sparse intersections (\cref{def:sparse-intersections}) carries over
unchanged, we are ready to prove the following analogue of
\cref{thm:decomp-to-classic} for weighted subproblem decompositions.

\begin{theorem}
  \label{thm:weighted-decomp-to-classic}
  Let $\mathcal F$ be a family of weighted subproblem decompositions
  that allows partial assignment checking. If $\mathcal F$ has sparse
  intersections, then we can in polynomial time reduce any weighted
  subproblem decomposition $S \in \mathcal F$ to a classic weighted
  CSP instance $P$ with $\hyp(P) = \hyp(S)$, such that $P$ has a
  solution with cost at most $k \in \mathbb Q$ if and only if $S$
  does.
\end{theorem}
\begin{proof}
  Let $S$ be a subproblem decomposition from $\mathcal F$. For each $T
  \in S$, $P$ will contain the table constraint $\ic(T)$ from
  \cref{def:ind-const}. Since $\mathcal F$ allows partial assignment
  checking and has sparse intersections, computing $\ic(T)$ can be
  done in polynomial time by invoking \cref{alg:enum-solutions} on
  $\Pj_{\iv(T)}(S)$.

  It is clear that $\hyp(P) = \hyp(S)$. All that is left to show is
  that $P$ has a solution with cost at most $k \in \mathbb N$ if and
  only if $S$ does. Let $\theta$ be a solution to $S$. For every $T
  \in S$, $\theta|_{\iv(T)} \in \Pj_{\iv(T)}(S)$ by
  \cref{def:weighted-constraint-projection,def:intersection-vertices},
  so the assignment $\mu$ that assigns the value $\theta(v)$ to each
  $v \in \displaystyle\bigcup_{T \in S}\iv(T)$, and $*$ to every other
  variable is a solution to $P$. Furthermore, for every $T \in S$ we
  have by \cref{def:weighted-ind-const} that $\cost(\ic(T),
  \mu|_{\vars(T)}) = \cost(\Pj_{\iv(T)}(T), \mu|_{\iv(T)})$, so by
  \cref{def:weighted-constraint-projection} $\cost(\ic(T),
  \mu|_{\vars(T)}) \leq \cost(T, \theta|_{\vars(T)})$ and therefore
  $\cost(P, \mu) \leq \cost(S, \theta)$.

  In the other direction, if $\theta$ is a solution to $P$, then
  $\theta$ satisfies $\ic(T)$ for every $T \in S$. By
  \cref{def:weighted-ind-const}, this means that $\theta|_{\iv(T)} \in
  \sol(\Pj_{\iv(T)}(S))$, and by
  \cref{def:weighted-constraint-projection}, there exists an
  assignment $\mu^T$ with $\mu^T|_{\iv(T)} = \theta|_{\iv(T)}$ that
  satisfies $T$, such that $\cost(\ic(T), \theta|_{\vars(T)}) =
  \cost(T, \mu^T)$. By \cref{def:intersection-vertices}, the variables
  not in $\iv(T)$ do not occur in any other subproblem from $S$, so we
  can combine all the assignments $\mu^T$ to form a solution $\mu$ to
  $S$ such that for $T \in S$ and $v \in \vars(T)$ we have $\mu(v) =
  \mu^T(v)$, with $\cost(P, \theta) = \cost(S, \mu)$. \qed
\end{proof}

As before, for a family of weighted subproblem decompositions
$\mathcal F$ we define $\WCSP(\mathcal F) = \{\bigsqcup S \mid S \in
\mathcal F\}$, and for a class of hypergraphs $\mathcal H$ we let
$\WCSP(\mathcal H, \Extlang)$ be the class of classic WCSP instances
whose hypergraphs are in $\mathcal H$. With that in mind, we can use
\cref{thm:weighted-decomp-to-classic} to obtain new tractable and
fixed-parameter tractable classes of weighted CSP instances with
global constraints.

\begin{corollary}
  \label{corollary:weighted-subproblem-csp}
  Let $\mathcal F$ be a family of weighted subproblem decompositions
  that allows partial assignment checking and has sparse
  intersections. If $\WCSP(\hyp(\mathcal F), \Extlang)$ is tractable
  or in \FPT{}, then so is $\WCSP(\mathcal F)$.
\end{corollary}
\begin{proof}
  Let $\mathcal F$ be given. By \cref{thm:weighted-decomp-to-classic},
  we can reduce any weighted subproblem decomposition $S \in \mathcal
  F$ to an instance $P \in \WCSP(\hyp(\mathcal F), \Extlang)$ in
  polynomial time. Since $P$ has a solution with cost $k$ if and only
  if $S$ does, tractability of $\WCSP(\hyp(\mathcal F), \Extlang)$
  implies the same for $\WCSP(\mathcal F)$. \qed
\end{proof}

\section{Summary}

We have studied the tractability of CSPs with global constraints under
various structural restrictions such as tree and hypertree width. By
exploiting the number of solutions to CSP instances in key places, we
have identified new tractable classes of such problems, both in the
ordinary and weighted case.

Furthermore, we have shown how this technique can be used to combine
CSP instances drawn from known tractable classes, extending a previous
result by Cohen and Green \cite{guarded-decomp-CG}. We have also shown
how the existence of back doors in CSP instances can be used to
augment our results.

More work remains to be done on this topic. In particular,
investigating whether a refinement of the conditions we have
identified can be used to show dichotomy theorems, similar to those
known for certain kinds of constraints and structural restrictions
\cite{ChenGroheCompactRel,grohe-hom-csp-complexity,DBLP:journals/toc/Marx10}. Also
of interest is the complexity of checking whether a constraint has few
solutions, which ties into finding classes of CSP instances that
satisfy \cref{def:sparse-intersections}.

\begin{acknowledgements}
  This work has been supported by the Research Council of Norway
  through the project DOIL (RCN project \#213115). The author thanks
  the anonymous reviewers for their detailed feedback.
\end{acknowledgements}

\bibliographystyle{spmpsci}
\bibliography{mainrefs-thesis}

\end{document}